\documentclass[11pt]{article}
\usepackage{amsmath,amssymb,amsthm,mathtools}
\usepackage{hyperref}
\usepackage{geometry}
\title{Path-Sampled Integrated Gradients}
\author{%
Firuz Kamalov\\
\small Department of Computational Sciences, Canadian University Dubai, UAE
\and
Fadi Thabtah\\
\small Abu Dhabi School of Management, Abu Dhabi, UAE
\and
R. Sivaraj\\
\small Department of Mathematics, Dr. B.R. Ambedkar National Institute of Technology, India
\and
Neda Abdelhamid\\
\small Abu Dhabi School of Management, Abu Dhabi, UAE
}
\date{\today}

\theoremstyle{definition}
\newtheorem{definition}{Definition}[section]
\theoremstyle{plain}
\newtheorem{theorem}[definition]{Theorem}
\newtheorem{proposition}[definition]{Proposition}
\newtheorem{corollary}[definition]{Corollary}

\theoremstyle{remark}
\newtheorem{remark}[definition]{Remark}

\begin{document}
\maketitle

\begin{abstract}
We introduce path-sampled integrated gradients (PS-IG), a framework that generalizes feature attribution by computing the expected value over baselines sampled along the linear interpolation path. We prove that PS-IG is mathematically equivalent to path-weighted integrated gradients, provided the weighting function matches the cumulative distribution function of the sampling density. This equivalence allows the stochastic expectation to be evaluated via a deterministic Riemann sum, improving the error convergence rate from $O(m^{-1/2})$ to $O(m^{-1})$ for smooth models. Furthermore, we demonstrate analytically that PS-IG functions as a variance-reducing filter against gradient noise---strictly lowering attribution variance by a factor of $1/3$ under uniform sampling---while preserving key axiomatic properties such as linearity and implementation invariance.
\end{abstract}


\section{Introduction}

The remarkable performance of deep neural networks  in domains ranging from medical diagnosis to autonomous driving has been accompanied by an increasing demand for transparency. To bridge the gap between model accuracy and interpretability, a variety of {feature attribution} methods have been developed to quantify the contribution of individual input features to a model's prediction \cite{lundberg2017unified, simonyan2013deep}. This objective shares a strong theoretical foundation with mathematical feature selection methods \cite{kamalov2025feature}, which similarly aim to isolate significant variables within complex systems. Among these, {integrated gradients (IG)} \cite{sundararajan2017axiomatic} has emerged as a gold standard due to its solid theoretical foundation. In particular, IG satisfies desirable axiomatic properties such as {sensitivity} and {completeness}, which many heuristic methods fail to uphold.

Despite its theoretical elegance, standard IG suffers from significant practical limitations, most notably its dependence on the choice of the baseline \cite{sturmfels2020visualizing}. The standard practice of using a zero-vector -- such as a a black image in computer vision tasks -- often introduces artifacts or fails to capture feature relevance when the baseline is far from the data manifold \cite{kapishnikov2021guided}. Furthermore, gradients in deep networks can be highly volatile; integrating along a single fixed path can aggregate noise, resulting in unstable or ``speckled'' attribution maps that degrade user trust \cite{smilkov2017smoothgrad}.

To address these shortcomings, the research community has proposed several extensions. {Expected gradients} (EG) \cite{erion2021improving} generalizes IG by averaging attributions over a distribution of baselines sampled from the training set, rather than relying on a single reference point. While effective at reducing baseline bias, EG is computationally expensive, requiring many forward and backward passes. Concurrently, methods like {weighted integrated gradients} \cite{kien2025weighted} and {path-weighted integrated gradients} (PWIG) \cite{kamalov2025path} have introduced mechanisms to modulate the integration path. Kien et al. propose weighting baselines based on a fitness function, while Kamalov et al. introduce a scalar weighting function $g(\alpha)$ into the path integral to emphasize specific segments of the interpolation. While these methods offer improvements in flexibility and stability, a unified framework that reconciles the stochastic nature of baseline sampling with the efficiency of deterministic path integration is still missing.

In this paper, we introduce \textit{path-sampled integrated gradients} (PS-IG), a novel formulation that addresses the fragility of single-baseline attribution while maintaining high computational efficiency. Instead of relying on a single starting point or sampling globally from a dataset, PS-IG computes the expected attribution over a distribution of baselines sampled strictly along the interpolation path between an initial reference and the input. The geometric and statistical properties of sampling along such linear segments have been previously analyzed in the context of synthetic data generation \cite{kamalov2025smote, kamalov2024smote}, suggesting that constrained linear sampling effectively captures the local manifold structure required for robust estimation."

We provide a theoretical analysis showing that PS-IG unifies two seemingly distinct approaches. First, we demonstrate that PS-IG can be viewed as a specific instance of EG, where the prior distribution is constrained to the linear path. Second, and more importantly, we prove that PS-IG is mathematically equivalent to PWIG, where the weighting function $g(\alpha)$ is the CDF of the sampling density.

This equivalence yields a ``free lunch'' in terms of computational efficiency: we can achieve the robustness benefits of stochastic sampling using a deterministic, weighted integral that costs no more to compute than standard IG. By formally linking sampling strategies to path weighting, PS-IG offers a robust, efficient, and theoretically grounded solution to the baseline selection problem in feature attribution.

\section{Background}

Let $F : \mathbb{R}^n \to \mathbb{R}$ be a differentiable model. Fix an input $x \in \mathbb{R}^n$ and a baseline $x' \in \mathbb{R}^n$. Denote the straight-line path from $x'$ to $x$ by
\[
\gamma(\alpha) = x' + \alpha(x - x'), \quad \alpha \in [0, 1].
\]
We write the path derivative 
\[
F'(\alpha) := \frac{d}{d\alpha}F(\gamma(\alpha)) = \sum_{i=1}^n (x_i - x'_i) \frac{\partial F(\gamma(\alpha))}{\partial x_i},
\]
where $x_i$ is an individual component of $x$.

\begin{definition}
\cite{sundararajan2017axiomatic}
The \textit{integrated gradients}  attribution for feature $i$ at input $x$ with baseline $x'$ is
\[
IG_i(x; x') := (x_i - x'_i) \int_0^1 \frac{\partial F(\gamma(\alpha))}{\partial x_i} \, d\alpha.
\]
\end{definition}

IG has several desirable properties when $F$ is smooth along the path: implementation invariance, sensitivity, symmetry preservation, and completeness.

\begin{definition}\cite{kamalov2025path}
Let $g : [0, 1] \to \mathbb{R}^+$ be an integrable weight function. The \textit{path-weighted integrated gradients} attribution for feature $i$ is
\[
PWIG_i(x; x'; g) := (x_i - x'_i) \int_0^1 g(\alpha) \frac{\partial F(\gamma(\alpha))}{\partial x_i} \, d\alpha.
\]
\end{definition}

\begin{remark}
When $g(\alpha) \equiv 1$ we recover standard IG. The factor $g(\alpha)$ allows emphasizing or de-emphasizing contributions from different segments of the path.
\end{remark}

Not that summing over $i$ gives
\[
\sum_{i=1}^n PWIG_i(x; x'; g) = \int_0^1 g(\alpha) F'(\alpha) \, d\alpha.
\]
Define the completeness residual
\[
R(g) := \Delta F - \sum_{i=1}^n PWIG_i(x; x'; g), \quad \Delta F := F(x) - F(x').
\]
Then, PWIG is complete iff $R(g) = 0$.
Note that
\[
R(g) = \int_0^1 (1 - g(\alpha)) F'(\alpha) \, d\alpha.
\]
Therefore, PWIG is complete iff $g(\alpha) \equiv 1$ or $F'(\alpha)$ is supported only where $g(\alpha) = 1$; in general completeness fails. Assume $g$ is continuously differentiable and $F(\gamma(\alpha))$ is continuous. Integration by parts yields
\[
\int_0^1 g(\alpha)F'(\alpha) \, d\alpha = g(1)F(1) - g(0)F(0) - \int_0^1 g'(\alpha)F(\gamma(\alpha)) \, d\alpha,
\]
so the residual can be written
\begin{equation} \label{eq:residual}
R(g) = \Delta F - \Big(g(1)F(1) - g(0)F(0)\Big) + \int_0^1 g'(\alpha)F(\gamma(\alpha)) \, d\alpha,
\end{equation}
where $F(s) = F(\gamma(s))$ in the boundary terms.

\section{Path-Sampled Integrated Gradients}

Let $p$ be a probability density on $[0, 1]$. Fix an input $x \in \mathbb{R}^n$ and an {initial baseline} $x' \in \mathbb{R}^n$. For each $s \in [0, 1]$, we define an intermediate baseline $b_s$ located on the straight-line path between the initial baseline and the input:
\[
b_s := x' + s(x - x').
\]
We compute the standard IG attribution for $x$ using the intermediate baseline $b_s$:
\[
IG_i(x; b_s) = (x_i - b_{s,i}) \int_0^1 \frac{\partial F(b_s + u(x - b_s))}{\partial x_i} \, du.
\]

\begin{definition}
\label{def:psig}
The \textit{path-sampled integrated gradients}  attribution for feature $i$, given input $x$, initial baseline $x'$, and sampling density $p$, is defined as:
\[
PSIG_i(x; x'; p) := \mathbb{E}_{s \sim p} \big[ IG_i(x; b_s) \big] = \int_0^1 IG_i(x; x' + s(x - x')) \, p(s) \, ds.
\]
\end{definition}

In the following result, we show that PS-IG can be simplified to a single path integral using the cumulative distribution function of the sampling density. In other words, PS-IG  is equivalent to $PWIG(\cdot, \cdot; G)$ with weight function $g = G$.

\begin{theorem}
\label{prop:psig_simplification}
Let $p$ be an integrable density on $[0, 1]$ and let $G$ be its CDF.
Then for each feature $i$,
\[
PSIG_i(x; x'; p) = (x_i - x'_i) \int_0^1 G(\alpha) \frac{\partial F(\gamma(\alpha))}{\partial x_i} \, d\alpha.
\]

\end{theorem}

\begin{proof}
Fix $s \in [0, 1]$. First, we derive two basic identities using the definition of $b_s$:
\begin{equation}
x_i - b_{s,i} = x_i - (x'_i + s(x_i - x'_i)) = (1 - s)(x_i - x'_i). \tag{2}
\end{equation}
Then, for the path integration variable $u$:
\begin{equation}
\begin{aligned}
b_s + u(x - b_s) &= x' + s(x - x') + u(1 - s)(x - x') \\
&= x' + (s + u(1 - s))(x - x') \\
&= \gamma(s + u(1 - s)).
\end{aligned} \tag{3}
\end{equation}
Thus,
\[
IG_i(x; b_s) = (1 - s)(x_i - x'_i) \int_0^1 \frac{\partial F(\gamma(s + u(1 - s)))}{\partial x_i} \, du.
\]
Change variables $\alpha = s + u(1 - s)$; then $du = d\alpha/(1 - s)$ and $\alpha$ runs from $s$ to $1$. Therefore,
\[
IG_i(x; b_s) = (x_i - x'_i) \int_{s}^1 \frac{\partial F(\gamma(\alpha))}{\partial x_i} \, d\alpha.
\]
Taking the expectation over $s$ and exchanging integrals (Fubini/Tonelli), we obtain
\begin{equation}
\begin{aligned}
\int_0^1 p(s) \left( \int_s^1 \frac{\partial F(\gamma(\alpha))}{\partial x_i} \, d\alpha \right) ds &= \int_0^1 \left( \int_0^\alpha p(s) \, ds \right) \frac{\partial F(\gamma(\alpha))}{\partial x_i} \, d\alpha \\
&= \int_0^1 G(\alpha) \frac{\partial F(\gamma(\alpha))}{\partial x_i} \, d\alpha.
\end{aligned} \tag{4}
\end{equation}
Multiplying by $(x_i - x'_i)$ concludes the proof.
\end{proof}

\begin{corollary}[Uniform sampling]
If $p \equiv 1$, then $G(\alpha) = \alpha$ and
\[
PSIG_i(x; x'; \text{Unif}) = (x_i - x'_i) \int_0^1 \alpha \frac{\partial F(\gamma(\alpha))}{\partial x_i} \, d\alpha.
\]
\end{corollary}

\begin{remark}
The family of PS-IGs obtained by sampling baselines on the straight-line path corresponds exactly to the family of PWIG weights $g$ that are nondecreasing CDFs with $g(0) = 0$ and $g(1) = 1$. Conversely, any PWIG whose weight $g$ is a CDF arises as a PS-IG under sampling density $p = g'$.
\end{remark}

\begin{remark}
\textit{Discrete sampling:} if $s_1, \dots, s_m$ are sampled i.i.d. from $p$, the empirical average
\[
\frac{1}{m} \sum_{j=1}^m IG_i(x; b_{s_j}) = (x_i - x'_i) \int_0^1 \widehat{G}_m(\alpha) \, h_i(\alpha) \, d\alpha,
\]
where $\widehat{G}_m$ is the empirical CDF of the $s_j$; as $m \to \infty$, $\widehat{G}_m \to G$ a.s.
\end{remark}

For PS-IG, substituting into Equation 1 we obtain the completeness residual:
\[
R(g) = \Delta F - F(1) + \int_0^1 g'(\alpha) F(\gamma(\alpha)) \, d\alpha.
\]
Since $\Delta F = F(1) - F(0)$, this simplifies to
\[
R(g) = -F(0) + \int_0^1 g'(\alpha) F(\gamma(\alpha)) \, d\alpha.
\]
Equivalently, writing $p(\alpha) = g'(\alpha)$,
\[
R(g) = \mathbb{E}_{s \sim p}[F(\gamma(s))] - F(x'),
\]
which shows that the residual is exactly the difference between the average prediction of $F$ along the path and the initial baseline prediction $F(x')$.

\subsection{Axiomatic Characterization}

All the axioms—implementation invariance, sensitivity (dummy), linearity, symmetry—that apply to PWIG are preserved for PS-IG.
On the other hand, it does not satisfy the standard completeness axiom. However, it satisfies a natural variation suited for stochastic baselines.

\begin{proposition}[Completeness]
\label{thm:expected_completeness}
Let $F: \mathbb{R}^n \to \mathbb{R}$ be differentiable. The sum of PS-IG attributions equals the difference between the model output at $x$ and the expected model output along the baseline path.
\[
\sum_{i=1}^n PSIG_i(x; x'; p) = F(x) - \mathbb{E}_{s \sim p}\left[ F(x' + s(x-x')) \right].
\]
\end{proposition}

\begin{proof}
By Definition 2.1 and the linearity of the expectation and summation:
\[
\begin{aligned}
\sum_{i=1}^n PSIG_i(x; x'; p) &= \sum_{i=1}^n \mathbb{E}_{s \sim p} \left[ IG_i(x; b_s) \right] \\
&= \mathbb{E}_{s \sim p} \left[ \sum_{i=1}^n IG_i(x; b_s) \right].
\end{aligned}
\]
Applying the completeness axiom of standard IG with respect to the baseline $b_s$, we have $\sum_{i=1}^n IG_i(x; b_s) = F(x) - F(b_s)$. Substituting this into the expectation:
\[
\begin{aligned}
\mathbb{E}_{s \sim p} \left[ F(x) - F(b_s) \right] &= F(x) - \mathbb{E}_{s \sim p} \left[ F(b_s) \right] \\
&= F(x) - \mathbb{E}_{s \sim p} \left[ F(x' + s(x-x')) \right].
\end{aligned}
\]
\end{proof}

\begin{remark}
The term $\mathbb{E}_{s \sim p}[F(b_s)]$ can be interpreted as a "smoothed baseline" prediction. Unlike standard completeness, which is sensitive to the exact value of $F(x')$, expected baseline completeness is robust to local fluctuations of the model at the initial baseline $x'$.
\end{remark}

\subsection{Computational Complexity and Efficiency}

A significant advantage of PS-IG is that it admits a deterministic approximation with a convergence rate superior to standard Monte Carlo sampling. 
If PS-IG is computed as an expectation, it would require a Monte Carlo estimator which involves sampling $m$ random baselines $s_1, \dots, s_m$ from the distribution $p$ and averaging the standard IG computed for each baseline. By the Central Limit Theorem, the root mean square error of this Monte Carlo estimator converges at a rate of $O(m^{-1/2})$. This method is computationally expensive, as it requires computing a full path integral for every sampled baseline.

However, by leveraging the theoretical equivalence between PS-IG and PWIG (Proposition 2.2), we can avoid stochastic sampling entirely. Instead, we can approximate the single weighted path integral using a deterministic Riemann sum:
\[
\hat{I}_{CDF} = (x_i - x'_i) \frac{1}{m} \sum_{k=1}^m G\left(\frac{k}{m}\right) \frac{\partial F(\gamma(k/m))}{\partial x_i}.
\]
For continuously differentiable functions, this deterministic estimator achieves an error convergence rate of $O(m^{-1})$, which is significantly faster than the Monte Carlo rate. The computational cost of this deterministic approximation is identical to that of standard IG. Thus PS-IG achieves the robustness benefits of an expectation-based method without incurring the multiplicative computational overhead typically associated with such methods.

\subsection{Stability and Variance Reduction}

A well-documented challenge in attributing deep neural networks is the phenomenon of shattered gradients \cite{balduzzi2017shattered}, where the gradient $\nabla F$ fluctuates rapidly and noisily as a function of the input. Standard IG integrates these fluctuations directly with uniform weight, making the resulting attribution sensitive to high-frequency noise along the path.

In this section, we prove that PS-IG is inherently more stable than standard IG. By modeling the gradient fluctuations as a stochastic noise process, we show that the weighting function $G(\alpha)$ inherent to PS-IG acts as a variance-reducing filter.

Let us model the gradient along the path $\gamma(\alpha)$ as the sum of a smooth signal component $\mu(\alpha)$ and a zero-mean, uncorrelated noise component $\xi(\alpha)$  representing the local volatility of the decision boundary. In particular, we make the following assumption.
For a feature $i$, the path gradient is given by:
\[
\frac{\partial F(\gamma(\alpha))}{\partial x_i} = \mu_i(\alpha) + \xi_i(\alpha),
\]
where $\mu_i(\alpha)$ is a deterministic signal and $\xi_i(\alpha)$ is a stochastic process satisfying:
\[
\mathbb{E}[\xi_i(\alpha)] = 0, \quad \mathbb{E}[\xi_i(\alpha)\xi_i(\beta)] = \sigma^2 \delta(\alpha - \beta),
\]
where $\sigma^2$ represents the noise magnitude and $\delta$ is the Dirac delta function.

We define the {attribution variance} as the variance of the attribution score induced by the gradient noise $\xi$. Lower variance implies an explanation that is more robust to the local volatility of the model.

\begin{theorem}
\label{thm:variance_reduction}
Let $Var_{IG}$ be the attribution variance of standard IG, and $Var_{PS}$ be the attribution variance of PS-IG.
If $p$ is a continuous density supported on $(0, 1)$, then:
\[
Var_{PS} = Var_{IG} \int_0^1 G(\alpha)^2 \, d\alpha.
\]
It follows that
\[
Var_{PS} < Var_{IG}. 
\]
\end{theorem}

\begin{proof}
The attribution $\mathcal{A}_i$ for a general weight $w$ is:
\[
\mathcal{A}_i = (x_i - x'_i) \int_0^1 w(\alpha) (\mu_i(\alpha) + \xi_i(\alpha)) \, d\alpha.
\]
The variance of this estimator is determined solely by the noise term:
\[
\text{Var}(\mathcal{A}_i) = (x_i - x'_i)^2 \text{Var} \left( \int_0^1 w(\alpha) \xi_i(\alpha) \, d\alpha \right).
\]
Using the isometry property of the stochastic integral for white noise (It\^o isometry), the variance is proportional to the squared $L^2$ norm of the weighting function:
\[
\text{Var} \left( \int_0^1 w(\alpha) \xi_i(\alpha) \, d\alpha \right) = \sigma^2 \int_0^1 w(\alpha)^2 \, d\alpha.
\]
For standard IG, $w(\alpha) = 1$:
\[
Var_{IG} = C \int_0^1 1^2 \, d\alpha = C, \quad \text{where } C = \sigma^2 (x_i - x'_i)^2.
\]
For PS-IG, $w(\alpha) = G(\alpha)$:
\[
Var_{PS} = C \int_0^1 G(\alpha)^2 \, d\alpha.
\]
Since $G$ is a CDF on $[0,1]$:
\[
\int_0^1 G(\alpha)^2 \, d\alpha <  1.
\]
Thus, $Var_{PS} < Var_{IG}$.
\end{proof}

The reduction in variance depends on the choice of the sampling strategy. We can quantify this gain for the most common sampling distribution.

\begin{corollary}
If baselines are sampled uniformly along the path, then:
\[
Var_{PS(\text{Unif})} = \frac{1}{3} Var_{IG}.
\]
\end{corollary}

This result lends support to the  smoothness in PS-IG attributions. Standard IG gives full weight to gradients along the entire path, including regions where the model might be uncertain or noisy. PS-IG dampens the contribution of gradients, particularly those near the baseline, effectively filtering out noise that accumulates early in the integration path.

\section{Numerical Experiments}
In this section, we demonstrate numerically the variance reduction and convergence rate advantages of the proposed PS-IG method.

To empirically validate the stability of PS-IG, we examined the attribution variance under a shattered gradient noise model. We utilized three 3-variable functions representing distinct gradient landscapes: linear, quadratic, and sigmoidal. For each function, we injected Gaussian white noise $\mathcal{N}(0, 1)$ into the gradients computed along the integration path  and performed 1,000 Monte Carlo trials.

The results, presented in Table \ref{tab:variance_results}, demonstrate that PS-IG significantly outperforms standard IG in terms of stability. Across all function types, PS-IG consistently reduces the empirical variance by a factor of approximately 3. This finding aligns strictly with the theoretical prediction of Corollary 2.3, confirming that PS-IG effectively filters high-frequency gradient noise regardless of the underlying model architecture.

\begin{table}[h]
\centering
\caption{Empirical variance of attribution scores under stochastic gradient noise. PS-IG consistently reduces variance by a factor of $\approx 1/3$ compared to standard IG.}
\label{tab:variance_results}
\begin{tabular}{l c c c}
\hline
{Function} & {Var(IG)} & {Var(PS-IG)} & {Ratio} \\
\hline
Linear ($x_1+x_2+x_3$)& 0.00992 & 0.00333 & 0.3356 \\
Quadratic ($x_1^2+x_1x_2+x_3^2$) & 0.01003 & 0.00335 & 0.3337 \\
Sigmoidal ($\sigma(10(\bar{x}-0.5))$) & 0.00982 & 0.00333 & 0.3394 \\
\hline
\end{tabular}
\end{table}

To test the computational advantage of the deterministic formulation, we analyzed the convergence rate of the PS-IG estimator. Using the 3-variable sigmoidal function from the previous experiment as the target, we computed the mean squared error (MSE) of the attribution against a high-precision ground truth ($N=10^5$) across varying computational budgets.
Figure \ref{fig:convergence} compares the deterministic approach against a Monte Carlo baseline. The results show that the deterministic estimator exhibits a significantly steeper convergence slope, achieving orders of magnitude lower error for the same number of gradient evaluations. 

\begin{figure}[h]
    \centering
    \includegraphics[width=0.9\linewidth]{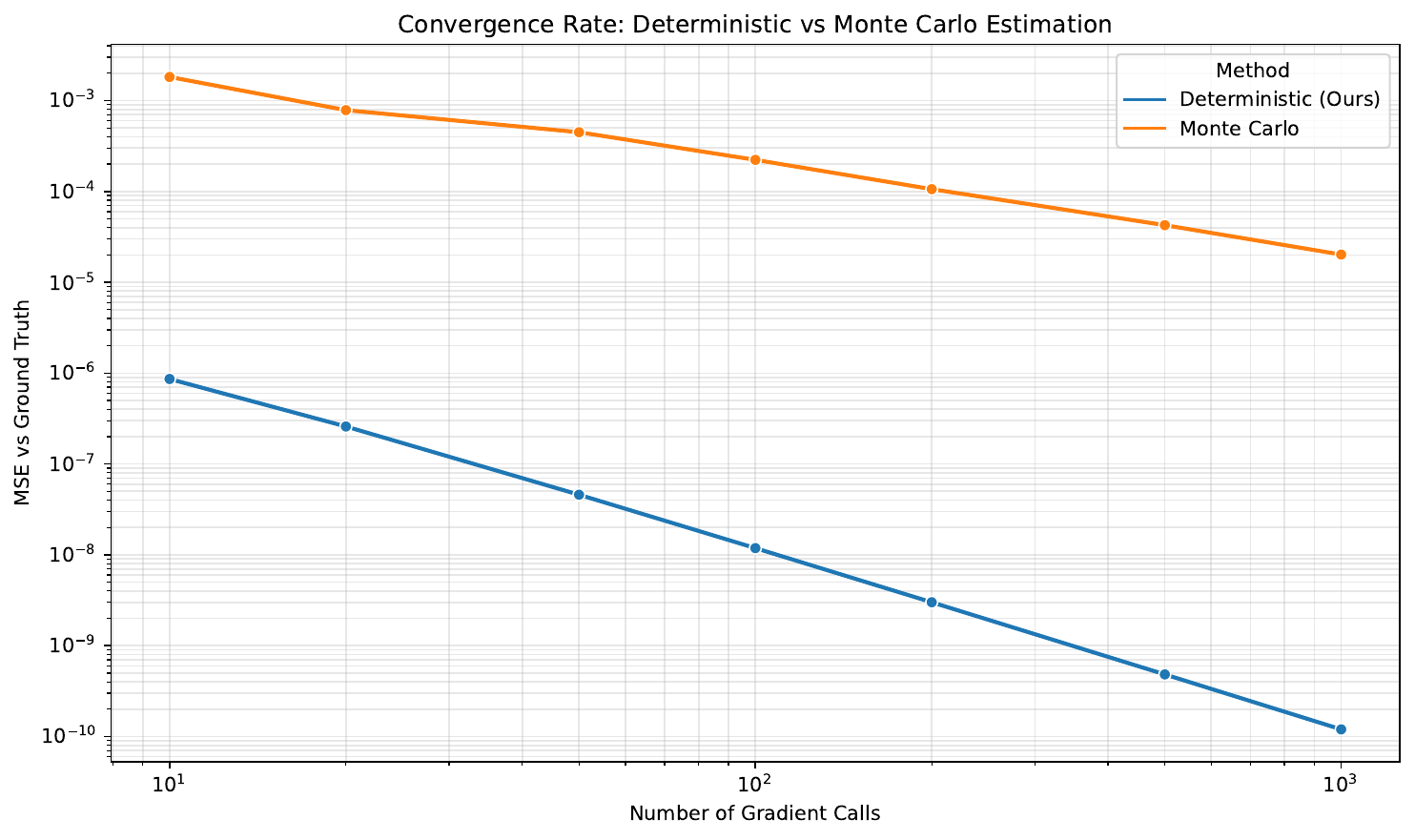} 
    \caption{Convergence rate comparison between the deterministic PS-IG estimator and Monte Carlo sampling on a log-log scale.}
    \label{fig:convergence}
\end{figure}

\section{Conclusion}
In this paper, we introduced PS-IG to address the baseline sensitivity and gradient volatility inherent in standard feature attribution methods. We established a formal equivalence between stochastic baseline sampling along a linear path and deterministic path weighting, proving that PS-IG is a specific instance of PWIG. This theoretical link allows for a deterministic implementation that improves the error convergence rate from $O(m^{-1/2})$ to $O(m^{-1})$ compared to Monte Carlo estimation. Furthermore, our variance analysis demonstrates that PS-IG inherently acts as a noise filter, strictly reducing attribution variance—specifically by a factor of $1/3$ under uniform sampling—without incurring additional computational costs. By reconciling the robustness of expectation-based methods with the efficiency of deterministic integration, PS-IG offers a rigorously grounded framework for reliable model interpretability.

\end{document}